\documentclass{article}
\usepackage{spconf}
\usepackage{graphicx}
\usepackage{float}
\usepackage{mdframed}
\usepackage{color,verbatim}
\usepackage{multirow}
\usepackage[linesnumbered,ruled,vlined]{algorithm2e}
\usepackage{amsmath}
\usepackage{amsthm}
\usepackage{amssymb}
\usepackage{times}
\usepackage{epstopdf}
\usepackage[stable]{footmisc}
\usepackage{url}
\newtheorem{theorem}{Theorem}

\newtheorem{proposition}[theorem]{Proposition}

\newcommand{\rset}{\mathbf{R}}

\newcommand{\I}{\mathcal{I}}
\newcommand{\J}{\mathcal{J}}
\newcommand{\A}{\mathcal{A}}

\newcommand{\N}{\mathcal{N}}

\providecommand{\norm}[1]{\lVert#1\rVert}

\newcommand{\be}{\begin{equation}}
\newcommand{\ee}{\end{equation}}

\newcommand{\bt}{\begin{tabular}}
\newcommand{\et}{\end{tabular}}

\newcommand\pdffig[4][7cm]{
	\begin{figure}[t]
		\centering
		\includegraphics[width=#1]{#2}
		\caption{#3}
		\label{#4}
	\end{figure}
}

\newlength{\algitab}
\title{Dictionary Learning with Uniform Sparse Representations for Anomaly Detection}
\name{Paul Irofti, Cristian Rusu, Andrei Pătrașcu
\thanks{
Paul Irofti was supported by a grant of the Romanian Ministry of Education and Research, CNCS - UEFISCDI,
project number PN-III-P1-1.1-PD-2019-0825, within PNCDI III. Andrei P\u atra\c scu was supported by a grant of the Romanian Ministry of Education and Research, CNCS - UEFISCDI, project number PN-III-P1-1.1-PD-2019-1123, within PNCDI III.
Paul Irofti and Andrei Pătrașcu were also supported by a grant of the Romanian Ministry of Education and Research, CNCS - UEFISCDI,
project number PN-III-P2-2.1-PED-2019-3248,
within PNCDI III. Cristian Rusu was supported by
		the Romanian Ministry of Education and Research, CNCS-UEFISCDI,
		project number PN-III-P1-1.1-TE-2019-1843, within PNCDI III.
}
}

\address{Research Center for Logic, Optimization and Security (LOS),
Department of Computer Science, \\
Faculty of Mathematics and Computer Science,
University of Bucharest, Romania}

\begin{document}
\maketitle
\begin{abstract}
Many applications like audio and image processing show that sparse representations are a powerful and efficient signal modeling technique. Finding an optimal dictionary that generates at the same time the sparsest representations of data and the smallest approximation error is a hard problem approached by dictionary learning (DL). We study how DL performs in detecting abnormal samples in a dataset of signals. In this paper we use a particular DL formulation that seeks uniform sparse representations model to detect the underlying subspace of the majority of samples in a dataset, using a K-SVD-type algorithm. Numerical simulations show that one can efficiently use this resulted subspace to discriminate the anomalies over the regular data points. 
\end{abstract}

\begin{keywords}
anomaly detection, dictionary learning, sparse representation
\end{keywords}
\section{Introduction}

Dictionary learning (DL) is a decomposition method with many applications to audio and image processing, compression, classification, and computer vision, where it gives better performance than popular transforms. Given the training data, DL builds a dictionary and sparse representations corresponding to data points by minimization of the approximation error, imposing the desired limits on coefficients sparsity.

Intuitively, the generic anomaly detection~(AD) problem consists of finding particular points in a given dataset, called anomalies or outliers, that are not conformal to the majority of the rest of the data points (called inliers).

DL algorithms construct representation vectors that have an unstructured support distribution, i.e., the sparsity pattern is unstructured and uses many of the theoretical possible subspaces to represent the data. When done this way, there is no indication of the existence of a common subspace that generates all the training data. But, we now assume that the regular majority of signals is generated by the same atoms.



\vspace{5pt}

\noindent \textbf{Prior work}. The idea of enforcing uniform support representations is not new, for instance in \cite{NieHua:10,Lan:18} it is used in feature selection problems in combination with squared Euclidean loss and other various robust losses. The work in \cite{TROPP2006572} introduces the Simultaneous Orthogonal Matching Pursuit (S-OMP) method to solve the sparse approximation step while also balancing the number of elementary signals that are used in the representations.
The Joint Sparse Representation (JSR) model analyzed in \cite{ZhaLi:13,LiZha:15} 
assumes a multi-class partitioning of the input, where each class spans a low-dimensional subspace. 
In \cite{AdlEla:15}, $\ell_1$ penalty-based JSR is used for detecting noisy anomalies with prefixed dictionary. Although their formulations are similar to ours, the authors aim to compute only specific Sparse Representations that highlight diversity in the dataset.


Some references on DL-based AD in images are listed further,
however note that these do not tackle the standard DL problem and thus do not generalize well.
In \cite{7026119}, local dictionaries with specific structures are enhanced based on information from neighbors for detecting abnormal images. The convolutional sparse coding model is exploited in  \cite{7280790} in order to learn a dictionary of filters used for the same task. 
Particular DL and SR formulations for AD in network traffic and telemetry are analyzed in \cite{10.1007/978-3-030-48256-5_34,Xing2020DetectingAI,pilastre:hal-02466360}. Empirical evidence on detection of anomalous images and electrocardiographic data are given in \cite{Andrysiak2018SparseRA,7008985,YuaDan:19,Xing2020DetectingAI,8237307}. 









\vspace{5pt}
\noindent \textbf{Contribution}.
In this paper, we first reformulate the DL problem and add a row sparsity regularization, by replacing the usual column sparsity penalty. We approach two particular penalties, $\ell_1$ norm and $\ell_0$-"norm", aiming to enforce an entire row of representations matrix $X$ to be null
while at the same time allowing subspace differentiation within the remaining rows.
Prior work uses regularization and other techniques to impose row sparsity in $X$ such that the final representations lie on the same subspace;
our work also imposes row sparsity but gains a competitive edge by allowing subspace differentiation within the selected rows.
By taking advantage of the sum decomposition of the new proposed regularizer, we devise a K-SVD-type algorithm with similar complexity as the usual K-SVD iteration \cite{DumIro:17,AhaEla:06,Iro:20}.
We show-case the performance of our approach in Section \ref{sec:numexp}.


\vspace{5pt}

\noindent \textbf{Notations}. Denote: $x_i$ the $i$-th row, $x^j$ the $j$-th column in matrix $X$.
The ordered $i$-th left and right singular vectors of matrix $X$ are $u^i(X)$ and $v_i(X)$, respectively. The "norm" $\norm{\cdot}_{0}$ counts the number of nonzero elements of a vector.
We use $[n] = \{1, \cdots, n\}$ for some $n \ge 1$. The set of column normalized matrices is $\N_{m,n} = \{D \in \rset^{m \times n}\!:\! \norm{d^j}  =1, \forall j \in [n] \}$.


\section{Problem formulation}

Let $Y$ be the input data, the basic DL problem is:
\begin{align*}
    \min_{X,D \in \N_{m,n}} \; \norm{DX - Y}^2_F +\lambda \sum_{i = 1}^N \norm{x_i}_{0} 
\end{align*}
where $d^j$ is the $j$-th atom of the dictionary $D$ and $X$ is the representation matrix.
The above $\ell_0$ regularization promotes unstructured sparsity in the columns of $X$.
Unfortunately this
does not reveal any underlying joint properties of signals $Y$ (Fig.\ref{fig:matrices} left).
Existing methods replace this standard regularization
in order to promote a similar support among the columns of $X$
(Fig.\ref{fig:matrices} center).
Therefore, to benefit from both,
our aim is to preserve the sparsity pattern from both coordinates
(Fig.\ref{fig:matrices} right).
The general model of interest is:
\begin{align}\label{problem_of_interest}
    \min_{X,D \in \N_{m,n}}   \; \frac{1}{2}\norm{DX - Y}^2_F + \lambda \sum\limits_{i=1}^n \phi(\norm{x_i}_{2}). 
\end{align}
where $\phi$ is a sparse regularizer. Although a similar intuition is shared by JSR, which aims to obtain sparsity pattern as in the center of Fig. \ref{fig:matrices}, we argue later that our algorithm could preserve the sparsity on both coordinates, illustrated in Fig. \ref{fig:matrices} (right), by nature of K-SVD iteration.
For simplicity we further use notation $F(D,X):=\frac{1}{2}\norm{DX - Y}^2_F + \lambda \sum\limits_{i=1}^n \phi(\norm{x_i}_{2})$. Simultaneous minimization over subsets of $\{x_i\}_{i=0}^N$ that spans multiples rows makes the regularization hard even for convex $\phi$. Our further approach involves alternating minimization over one row at each iteration in order to obtain an algorithmic scheme with simple steps.

\section{Algorithms and Methodology}

The K-SVD algorithm introduced in \cite{AhaEla:06} selects at each iteration $k$ an index $i_k$ and, based on the information at previous step $k-1$, minimizes the residual over $d^{i_k}$ and $x_{i_k}$, while it keeps unchanged $d^j = (d^j)^k$ and $x_j = (x_j)^k$ for $j \neq i_k$.
The adaptation of this K-SVD reasoning to our regularized model
leads to the following: at iteration $k$, choose $i \in [N]$
\begin{align}\label{general_KSVD}
\!\!    (d^{i,k+1}\!,x_i^{k+1}) \!=\! \arg \!\!\! \min\limits_{\norm{d^{i}}=1,x_i} \! \frac{1}{2}\norm{d^{i}x_i \!-\! R^k}^2 \!+\! \lambda \phi(\norm{x_i}_2)
\end{align}
where $\{d^{j,k}, x_j^k \}$ are the $j-$th atom of the dictionary $D^k$ and $j-$th row of $X^k$, respectively. Also we denote $R^k = Y - \sum_{j \neq i_k} d^{j,k} x_{j}^k$. If $\lambda = 0$, then $(d^{i,k+1},x_i^{k+1})$
are the maximal left $u^1(R^k)$ and right $v_1(R^k)$ singular vectors. This iteration guarantees a decrease in the objective function $F$.

\begin{proposition} Let $\{D^k,X^k\}_{k \ge 0}$ be the sequence generated by the Algorithm 1. Then the following decrease hold:
$F(D^{k+1},X^{k+1}) \le F(D^{k},X^{k}) \qquad \forall k \ge 0.$
\end{proposition}
\begin{proof}
The iteration \eqref{general_KSVD} claims that $(d^{i,k+1},x_i^{k+1})$ is the minimizer of the right-hand side objective. Therefore, the value of this objective in $(d^{i,k+1},x_i^{k+1})$ is lower than its evaluation in the previous iterate. Thus, by using this fact we have:
\begin{align*}
& F(D^{k+1},X^{k+1})  = \frac{1}{2}\norm{d^{i,k+1}x_i^{k+1} - R^k}^2 \\
& \qquad \quad + \lambda \phi(\norm{x_i^{k+1}}_2) + \lambda \sum\limits_{j\neq i}\phi(\norm{x_j^{k}}_2) \\
& \le \frac{1}{2}\norm{d^{i,k}x_i^{k} - R^k}^2 + \lambda \sum\limits_{i}\phi(\norm{x_i^{k}}_2) = F(D^{k},X^{k}).
\end{align*}

\vspace{-0.5cm}
\end{proof}

\vspace{-10pt}

\noindent We further show that the above algorithm allows explicit forms of the solution
at each iteration  \eqref{general_KSVD} for some important cases when $\phi$ identifies with the most used sparse penalties.

\subsection{Convex $\ell_{2,1}$ regularization}	

Let $\phi(z) = z$, then in this particular case the regularizer of \eqref{problem_of_interest} becomes the widely known sparse penalty $\norm{X}_{2,1}$. 

\begin{proposition}\label{prop:l1}
Let $ \phi(x) = x $ and $\sigma_1$ be the maximal singular value of $R^k$, then the closed form solution of K-SVD iteration \eqref{general_KSVD} is:
if  $ \sigma_1 \ge \lambda$ then
\begin{align}\label{opt_l1}
	    (d^{i,k+1},x_i^{k+1})= \left(u^1(R^k), (\sigma_1 - \lambda) v_1(R^k)\right), 
\end{align}
otherwise $(d^{i,k+1},x_i^{k+1}) = (d^{i,k}, 0)$.
\end{proposition}

\begin{proof}
For simplicity, we redenote $d:= d^{i, k+1}, x:= x_i^{k+1}$ and $ t := \norm{x_i^{k+1}}$. The singular values of $R^k$ are called $\sigma_i$. We represent $d$ and $x$ in the SVD basis of $R^k$:
$d = \sum_{i=1}^m \rho_i u^i(R^k)$ and $  
x = \sum_{i=1}^N \theta_i v_i(R^k)$ 
to expand the objective in the new form:
\begin{equation*}
F(D^{k+1},X^{k+1})=\frac12 \norm{R^k}_F^2 - x^T(R^k)^Td + \frac12 \norm{x}^2 + \lambda \norm{x}.
\end{equation*}
Note that we can write $(R^k)^Td = \sum_{i=1}^r \sigma_i v_i(R^k) u^i(R^k)^T d = \sum_{i=1}^r \sigma_i \rho_i v_i(R^k)$.
Then the minimization of $F$ becomes:
\begin{equation}\label{problem_of_interest_transformed}
 \min_{t\ge 0, \theta,\rho} \;    \lambda t + \frac12 t^2 - \sum_{i=1}^r \sigma_i\rho_i\theta_i + \norm{R^k}_F^2.
\end{equation}
Further, observe that by the Cauchy-Schwarz inequality:
$\left(\sum_{i=1}^r \sigma_i\rho_i\theta_i\right)^2 \le \left(\sum_{i=1}^r \sigma_i\rho_i\right)^2 \left(\sum_{i=1}^N \theta_i\right)^2 \le
    \sigma_1^2 t^2$
yields that $(\theta^*,\rho^*) = (e_1,e_1)$ are optimal for any $t$ in problem \eqref{problem_of_interest_transformed}. Finally, the final form of \eqref{problem_of_interest_transformed} remains:
$\min_{t \ge 0} \;      \lambda t + \frac12 t^2 - \sigma_1(R^k) t + \norm{R^k}_F^2$
which has solution $t^* = \max\{0,\sigma_1 - \lambda \}$. Observe that for $\lambda > \sigma_1$, the optimal row $x_i^{k+1}$ is null.
\end{proof}


\subsection{Nonconvex  regularizers} 

\noindent There is wide evidence that nonconvex regularizers guarantees in some cases better performance, than convex ones, on unstructured sparse optimization problems \cite{ZhaLi:13}. We elaborate the explicit form of iteration \eqref{general_KSVD} when $\phi$ is the $\ell_0$ regularization, i.e. $\phi(x) = \norm{x}_0$. 
At iteration $k$, the index $i \in [n]$ is chosen and the following subproblem is solved:
\begin{align*}
    & (d^{i,k+1},x_i^{k+1}) = \arg\min\limits_{\norm{d^{i}} = 1,x_i} \; \frac{1}{2}\norm{d^{i}x_i - R^k}^2 + \lambda \norm{\norm{x_i}_2}_0
\end{align*}
\begin{proposition}\label{prop:l3}
Let $ \phi(x) =\norm{x}_0 $, then the closed form solution of K-SVD iteration \eqref{general_KSVD} is:
\begin{align}\label{opt_l0}
	    (d^{i,k+1},x_i^{k+1})= (u^1(R^k), v_1(R^k)), 
\end{align}
assuming $\frac{1}{2}\norm{u^1(R^k)v_1(R^k) - R^k}^2 \le \frac{1}{2}\norm{R^k}^2-\lambda$. Otherwise $(d^{i,k+1},x_i^{k+1}) = (d^{i,k}, 0)$.
\end{proposition}

\begin{proof}
Obviously, when $ \phi(x) = \norm{x}_0 $ there are only two possible cases: $(i)$ the solution $x_i^{k+1}$ is nonzero and, thus, is the same with the usual K-SVD update $\left(u^{1}(R^k), v_{1}(R^k)\right)$, see \cite{DumIro:17}; $(ii)$ $x_i^{k+1}$ is null when this value guarantees a larger descent on the local objective function, i.e.
$\frac{1}{2}\norm{R^k}^2 \le \frac{1}{2}\norm{u^{i}(R^k) v_{i}(R^k) - R^k}^2 + \lambda.$
\end{proof}

\vspace{-5pt}

\noindent The penalty parameter $\lambda$ represents the only degree of freedom that influences the number of 0-rows in the optimal $X^*$. Equivalently, large values of $\lambda$ yields an increasing number of ignored atoms in the final sparse representations of $Y$. Since the large energy of row $x_i$ might reflect a large importance of atom $d^i$, its elimination is undesirable. Thus, we might consider a truncated $\ell_2$ norm, that promotes sparsity on the low-norm rows in $X$. In this case, the particular penalty $\phi(z) = \ell_{\epsilon}(z):= \min\{|z|,\epsilon\}$ would penalize only the components that are below threshold $\epsilon$. Although iteration \eqref{general_KSVD}, with this form of $\phi$, keeps a simple and explicit form, our experiments did not showed any improvement over $\ell_1$ and $\ell_0$ models.


\begin{figure}[!t]
		\centering
		\includegraphics[trim = 48 650 50 50, clip, width=0.5\textwidth]{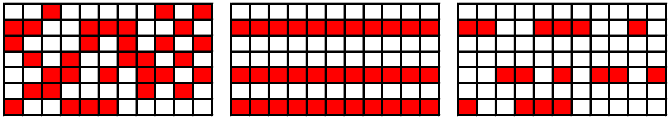}
		\caption{The non-zeros entries (red squares) in the sparse representation matrix $X$ for: the single measurement vector case (left), the multiple measurement vector -- also called the simultaneous -- case (center), and the proposed (right).}
		\label{fig:matrices}
	\end{figure}


\begin{algorithm}[t!]
\DontPrintSemicolon
\SetKwComment{Comment}{}{}
\KwData{train set $Y \in \rset^{m \times N}$,
test set $\Tilde{Y} \in \rset^{m \times \Tilde{N}}$,
$D^1 \in \rset^{m \times n}$,
sparsity $s$,
iterations K,
        }
\KwResult{anomalies $\A$}
\BlankLine

\emph{Training Procedure }\\
Representation: $X^1 = \text{OMP}(Y, D^1, s)$ \\
\For{$k \in \{1,\dots,K \}$}{
Error: $E^k = Y - D^k X^k$ \\
\For{$i \in \{1,\dots,n\}$}{
    Atom error: $R^k = E^k + d^{i,k} x_{i}^{k}$ \\
    SVD rank-1 approximation: $R^k \approx u^1\sigma_1 v_1$ \\
    K-SVD update: $(d^{i,k+1}_\text{SVD}, x_{i,\text{SVD}}^{k+1}) \!=\! (u^1, \sigma_1 v_1)$ \\
    Regularization: apply \eqref{opt_l1} or \eqref{opt_l0}\\ 
    New error: $E^{k} = R^k - d^{i,k+1} x_{i}^{k+1}$ \\
}
Uniform Support: $\I = \{i \mid \norm{x_i}_0 \ne 0\}$
}
\BlankLine
\emph{Anomaly Detection}\\
Representation: $X_\text{test} = \text{OMP}(\Tilde{Y}, D^{K+1}, s)$ \\
\For{$i \in [\Tilde{N}]$}{
    $\J = \{j \mid x_{\text{test},j}^i \ne 0\}$ \\
    \lIf{$\J \nsubseteq \I$}{$\A = \A \cup \{i\}$
    }
}
\caption{Uniform DL Representation for AD}
\label{alg:supp-ksvd}
\end{algorithm}
%
\begin{table*}[t!]

\tabcolsep 5pt

\caption{Maximum AD accuracy, standard deviation in parenthesis and running times for real datasets.}
\label{tab:realresults}
\small
\begin{center}
\bt{l l || c c | c c | 
c c | c c | c c | c c}
Dataset&($m$,$N$,outliers) & \multicolumn{2}{c|}{DL-$\ell_0 (\sigma)$} & \multicolumn{2}{c|}{DL-$\ell_1 (\sigma)$} & 
\multicolumn{2}{c|}{OC-SVM} & \multicolumn{2}{c|}{LOF} & \multicolumn{2}{c|}{IForest} \\
\hline
satellite&(36, 6435, 2036) & \textbf{0.8059}\ (0.059) & 0.57s & 0.8020\ (0.059) & 0.77s
& 0.6391 & 0.36s  & 0.5677 & 0.2s  & 0.7062 & 0.14s  \\
shuttle&(9, 49097, 3511) & 0.8155\ (0.109) & 1.16s & 0.9262\ (0.107) & 1.18s
& 0.6322 & 2.40s  & 0.5269 & 0.1s  & \textbf{0.9771} & 0.21s  \\
pendigits&(16, 6870, 156) & 0.7679\ (0.110) & 0.19s & \textbf{0.8822}\ (0.108) & 0.29s
& 0.7748 & 0.03s  & 0.5895 & 0.01s  & 0.8612 & 0.11s  \\
speech&(400, 3686, 61) & 0.5510\ (0.008) & 2.96s &0.5485\ (0.022) & 6.89s
& \textbf{0.5917} & 0.03s  & 0.5 & 0.02s  & 0.5289 & 0.2s  \\
mnist&(100, 7603, 700) & 0.5882\ (0.015) & 7.31s & \textbf{0.5917}\ (0.013)& 27.0s
& 0.5576 & 0.9s  & 0.5736 & 0.05s  & 0.5255 & 0.2s  \\
\et
\end{center}
\end{table*}

\subsection{The proposed algorithm}

The proposed procedure is given in Algorithm 1 with separate: (\textit{Training Procedure}) and ({\it Anomaly Detection}) sections.

In the training phase, we assume that an initial dictionary with $n$ atoms is available (possibly, a random dictionary) and that the dataset $Y$ was already split into the training and test sets of dimension $N$ and $\hat{N}$, respectively. The training set does not contain any anomalies. The first step of the algorithm is to construct the sparse representations via the
greedy Orthogonal Matching Pursuit (OMP)~\cite{PRK93omp} algorithm
that iteratively seeks a separate $s$-sparse representation for each signal in $Y$.
Based on the sparsity pattern computed by OMP, each one of the following iterations will update all the atoms of the dictionary by the standard K-SVD approach (an SVD step on the residual matrix $R^k$ but only on the columns that use the current atom) modified to take into account one of the regularizers we described in Section 3. The effect of the regularizer is to apply a joint sparsity constraint on the rows of $X^k$ while preserving the sparsity pattern 
of the remaining rows
(see Figure \ref{fig:matrices}) originally decided by the OMP algorithm. We highlight that OMP runs a single time, at the start of Algorithm 1 and not with every iteration. The support set $\mathcal{I}$ contains the indices of the rows from $X$ which are non-zero.

Once the training is complete and we have the dictionary $D$ and the set $\mathcal{I}$, in the AD step we use the OMP algorithm to compute the sparse representations $X_\text{test}$ on the test dataset $\hat{Y}$ and the we classify individually the data points as anomalies when in their sparse support there is a single atom from $D$ which is not in the set $\mathcal{I}$.
Note that standard DL algorithms,
such as K-SVD,
produce in general a uniformly distributed support across the representations $X$ which make them unfeasible for AD:
if $\I = [n]$ then
$\J$ would always be included in $\I$ at step 16
and thus no anomalies would be detected.

\section{Numerical experiments\footnotemark}
\label{sec:numexp}

\footnotetext{\noindent Python code at https://github.com/pirofti/AD-USR-DL}

In this section, we provide synthetic and real-world numerical experiments to evaluate the performance of the proposed algorithm. We also compare against some of the state-of-the-art methods from the AD literature,
but not with standard DL algorithms as they are not fit for AD.
Throughout our experiments we use $K=20$ and $s=0.2\sqrt{m}$ and use $90\%$ of the available inliers for training
(the outliers are not included).
We run on an AMD Ryzen Threadripper PRO 3955WX with 512GB of memory using Python 3.9.7 and Scikit-learn 1.0.

\pdffig[\columnwidth]
    {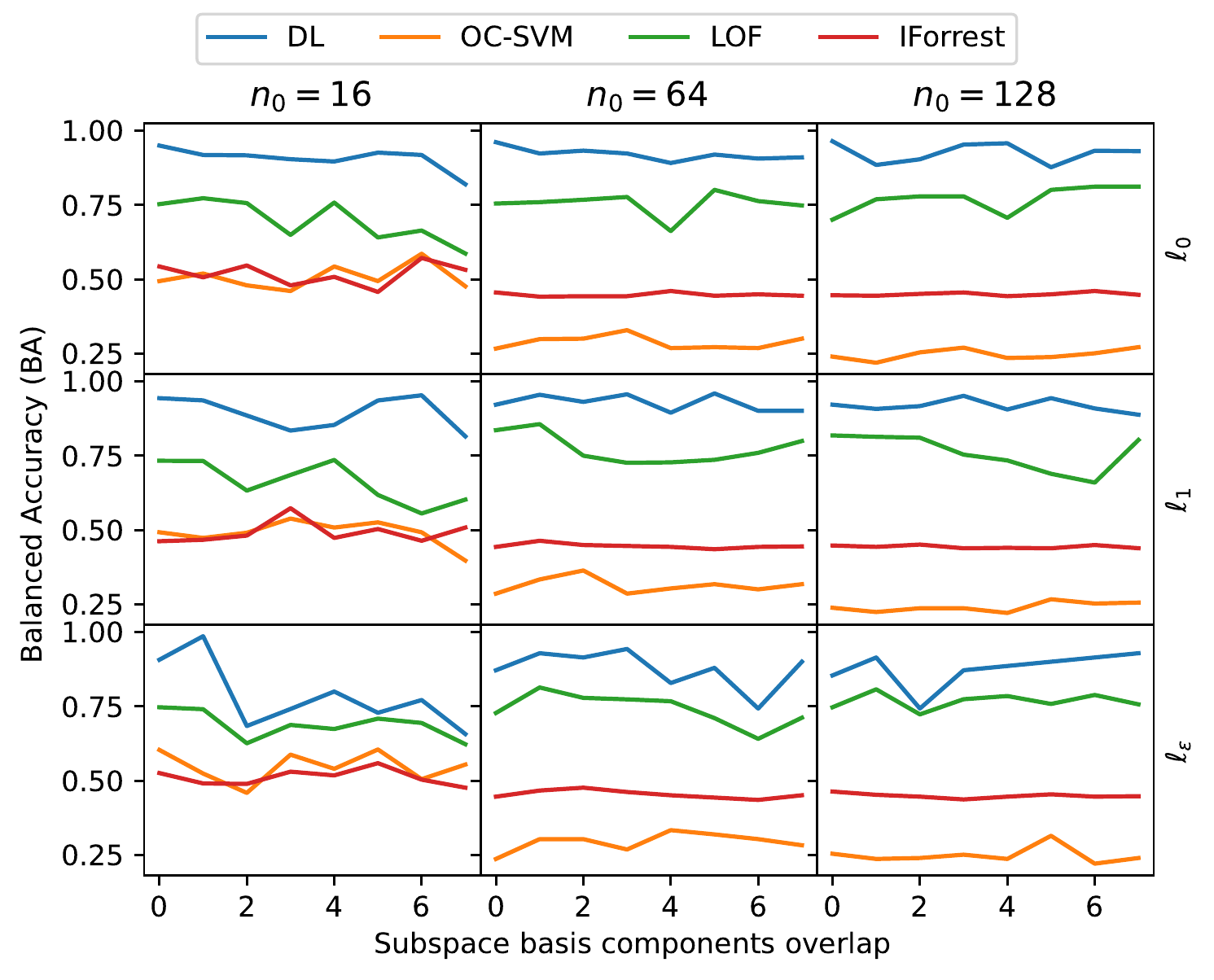}
    {Average accuracy of AD for synthetic data.}
    {fig:synth_accuracy}

Our first experiment is based on synthetic data. 
We generate two dictionaries of sizes $n_n$ for inliers and $n_0$ for outliers from which we generate $N+\Tilde{N}$ signals with $m=64$.
Each signal is produced by randomly choosing $s$ atoms from
one of the dictionaries
that produce a linear combination together with
their associated coefficients drawn from the normal distribution.
To harden the problem we also create an overlap between the atoms of the two dictionaries.
We split the resulting dataset into the training set $Y$,
and the testing data set
$\Tilde{Y}$ built from the outliers together with the remaining inliers.
In the paper the outliers represent $10\%$ of the total amount of testing signals.
We have tested with similar results anomaly planting from 1 to 20 percent.
DL starts with an $n=128$ normalized randomly generated dictionary ($n \gg \max\{n_n, n_0\}$).

Figure~\ref{fig:synth_accuracy} presents 9 rounds of experiments where we vary from $n_0=16$ to $n_0=128$ on the columns.
The rows represent experiments with the regularizations from Section 3.
Each plot presents results with different degrees of overlap between the original generating dictionaries and their effect on the balanced accuracy (BA).
We choose BA because it averages the sensitivity and specificity
of our models thus giving the reader a sense of both false positives and false negatives (the undetected anomalies).

For our second experiment,
we have chosen 5 datasets that belong to the publicly available Outlier Detection DataSets (ODDS)\footnote{http://odds.cs.stonybrook.edu/}.
We chose these datasets to span a wide range of available features and number of outliers. In all cases,
in the testing phase we use $10\%$ of the inlier data and all the outlier data.
We present the results in Table \ref{tab:realresults},
including the standard deviation $\sigma$ for our methods shown in parenthesis.
The proposed method,
with both regularizers $\ell_0$ and $\ell_1$,
performs best in 3 out of 5 cases and stays competitive in the other two
but has the largest running time among the methods we consider.
The parameter $\lambda$ of the proposed method is optimized by using a grid search.
The competing methods,
One Class - Support Vector Machine (OC-SVM)~\cite{ocsvm},
Local Outlier Factor (LOF)~\cite{lof},
and Isolation Forest~\cite{iforest},
were optimized through an extensive grid-search across multiple kernels, metrics and hyper-parameters
(OC-SVM did not always convergence on Shuttle and MNIST).

\section{Conclusions}


In this paper we propose a new dictionary learning based anomaly detection scheme
with uniform sparse representations.
Our algorithm starts with an initial sparse support
and proceeds only with regularized rank-1 update iterations. Avoiding sparse representation on each  dictionary learning iteration allows us to guarantee a descent on the objective function regularized by $\ell_{2,0}$,
$\ell_{2,1}$, 
and
$\ell_{2,\epsilon}$ penalties.
We also provide numerical experiments that confirm our method.

In the future we plan on providing an in-depth analysis of the parameters effect on the anomaly detection task
in order to establish a more rigorous detection scheme. 




\clearpage	
\bibliographystyle{IEEEbib}
\bibliography{factorization_ad}
	
\end{document}